\documentclass[lettersize,journal]{IEEEtran}
\usepackage{mathtools} %
\usepackage{booktabs} %
\usepackage{tikz} 
\usepackage{microtype}

\usepackage{subfigure}
\usepackage{soul}
\usepackage{amssymb}
\usepackage{amsthm}
\usepackage{physics}
\usepackage{enumitem}
\usepackage{multirow}
\usepackage{paracol}
\usepackage{multicol}
\usepackage{nicefrac} 
\usepackage{placeins}

\usepackage{amsmath,amsfonts}
\usepackage{algorithmic}
\usepackage{algorithm}
\usepackage{array}
\usepackage[caption=false,font=normalsize,labelfont=sf,textfont=sf]{subfig}
\usepackage{textcomp}
\usepackage{stfloats}
\usepackage{url}
\usepackage{verbatim}
\usepackage{graphicx}
\usepackage{cite}
\newtheorem{corollary}{Corollary}

\newtheorem{theorem}{Theorem}
\newtheorem{lemma}{Lemma}

\def \x {\mathbf{x}}

\def \W {\mathbf{W}}

\def \U {\mathbf{U}}

\begin{document}

\title{Delaytron: Efficient Learning of Multiclass Classifiers with Delayed Bandit Feedbacks}

\author{Naresh Manwani, Mudit Agarwal\\
Machine Learning Lab, IIIT Hyderabad, India\\
naresh.manwani@iiit.ac.in, mudit.agarwal@research.iiit.ac.in
}



\maketitle

\begin{abstract}
In this paper, we present online algorithm called {\it Delaytron} for learning multi class classifiers using delayed bandit feedbacks. The sequence of feedback delays $\{d_t\}_{t=1}^T$ is unknown to the algorithm. At the $t$-th round, the algorithm observes an example $\mathbf{x}_t$ and predicts a label $\tilde{y}_t$ and receives the bandit feedback $\mathbb{I}[\tilde{y}_t=y_t]$ only $d_t$ rounds later. When $t+d_t>T$, we consider that the feedback for the $t$-th round is missing. We show that the proposed algorithm achieves regret of $\mathcal{O}\left(\sqrt{\frac{2 K}{\gamma}\left[\frac{T}{2}+\left(2+\frac{L^2}{R^2\Vert \W\Vert_F^2}\right)\sum_{t=1}^Td_t\right]}\right)$ when the loss for each missing sample is upper bounded by $L$. In the case when the loss for missing samples is not upper bounded, the regret achieved by Delaytron is $\mathcal{O}\left(\sqrt{\frac{2 K}{\gamma}\left[\frac{T}{2}+2\sum_{t=1}^Td_t+\vert \mathcal{M}\vert T\right]}\right)$ where $\mathcal{M}$ is the set of missing samples in $T$ rounds. These bounds were achieved with a constant step size which requires the knowledge of $T$ and $\sum_{t=1}^Td_t$. For the case when $T$ and $\sum_{t=1}^Td_t$ are unknown, we use a doubling trick for online learning and proposed Adaptive Delaytron. We show that Adaptive Delaytron achieves a regret bound of $\mathcal{O}\left(\sqrt{T+\sum_{t=1}^Td_t}\right)$. We show the effectiveness of our approach by experimenting on various datasets and comparing with state-of-the-art approaches.
\end{abstract}

\begin{IEEEkeywords}
Online learning, multiclass classification, bandit feedback, delayed feedback, regret.
\end{IEEEkeywords}

\section{Introduction}
Learning classifiers using bandit feedback is a well-addressed problem in the machine learning community. The learner can only get feedback on whether the predicted label is correct in the bandit feedback setting. In \cite{10.1145/1390156.1390212}, the Banditron algorithm is proposed to learn with bandit feedback. When the data is linearly
separable, Banditron achieves a mistake bound of $\mathcal{O}(\sqrt{T})$ in the expected sense. In the general
case, Banditron makes $\mathcal{O}(T^{2/3})$ mistakes. On the other hand, Newtron \cite{Hazan2011NewtronAE}, which uses the online Newton method on a strongly convex objective function (adding regularization term with the loss function), achieves $\mathcal{O}(\log T)$ regret bound in the best case and $\mathcal{O}(T^{2/3})$ in the worst case. An exact passive-aggressive approach is proposed to learn classifiers in the bandit setting in \cite{pmlr-v129-arora20a}. A second-order algorithm is discussed in \cite{pmlr-v70-beygelzimer17a} which achieves $\tilde{\mathcal{O}}(\sqrt{T})$ regret. Banditron algorithm has been extended in the noisy bandit feedback setting in \cite{10.1007/978-3-030-75765-6_36} and shown to achieve $\mathcal{O}(T^{2/3})$ mistake bound in the worst case. A dilute bandit feedback setting is proposed in \cite{10.1007/978-3-030-89188-6_5} in which the algorithm predicts a subset of labels and gets the feedback on whether the actual label lies in that set. It is shown that even with this weaker bandit feedback, the algorithm can learn good classifiers \cite{10.1007/978-3-030-89188-6_5}.

The above approaches assume that the feedback for trial $t$ is received in the same round—however, it is not the case in many practical situations. 
Consider the example of sponsored advertising in which user's queries are answered as a list of product ads to maximize the clicks (by considering the query as well as the user's history and buying patterns) \cite{10.1145/2648584.2648589}. Users' clicks on a particular ad might indicate that the user is interested in that product. However, the user may take their own sweet time to click on an ad, and sometimes it may not. This causes a delay in the feedback on whether the user liked some ad or not. Meanwhile, the system still serves ads to other users and does not wait to know previous users' click outcomes. This situation also arises in the healthcare system where the feedback of a specific treatment on a patient can be received only after some time or delay \cite{Ong2018DelayIR}. This delay is natural as it takes time to observe the effects of the treatment. However, the healthcare professional still needs to treat other patients for the same problem before watching the treatment outcome of the previous patients. In finance, online learning algorithms which are used to manage portfolios always suffer from information and transaction delays from the market \cite{Lim2010TheDO}. Financial firms invest a massive amount of effort in minimizing these delays.

Adversarial multi arm bandits were first studied in \cite{pmlr-v49-cesa-bianchi16}
under the assumption of
fixed delay $d$. They gave a lower bound of $\Omega(\max(\sqrt{kT},\sqrt{dT\log k}))$ for $d \leq \frac{T}{\log k}$ and a matching upper bound of $\mathcal{O}\left(\sqrt{dT\log k}+\sqrt{kT\log k}\right)$. Bistritz et al. \cite{NEURIPS2019_ae2a2db4} proposed a variant of exp3 algorithm that can handle non-uniform delays and achieves $\mathcal{O}(\sqrt{kT\log(k)} + \sqrt{\sum_{t=1}^Td_t \log(k)})$ regret bound under the assumption that $T$ and $\sum_{t=1}^Td_t$ are known in advance. Bistritz et al. \cite{NEURIPS2019_ae2a2db4} provide a doubling scheme that achieves an $\mathcal{O}(\sqrt{k^2T\log(k)} + \sqrt{\sum_{t=1}^Td_t \log(k)})$ regret bound when $T$ and $\sum_{t=1}^Td_t$ are unknown. Stochastic gradient descent on convex objective function with delayed feedback is explored in \cite{NIPS2015_72da7fd6}. In \cite{NIPS2015_72da7fd6}, it is assumed that the convex function corresponding to time $t$ is revealed after a delay of $d_t$ only and proposed an approach which achieves a regret bound of $\mathcal{O}(\sqrt{\sum_{t=1}^Td_t})$.

There has been no work that deals with learning multiclass classifiers using delayed bandit feedback. This paper proposes an approach that efficiently learns classifiers using delayed bandit feedback, and this is the first work in this direction. Following are the key features of our proposed method.

There has been no work that deals with learning multiclass classifiers using delayed bandit feedback. This paper proposes an approach that efficiently learns classifiers using delayed bandit feedback. This is the first work in this direction. Following are the key features of our proposed approach.
\begin{enumerate}
    \item The proposed approach achieves a regret bound of $\mathcal{O}\left(\sqrt{\frac{2 K}{\gamma}\left[\frac{T}{2}+\left(2+\frac{L^2}{R^2\Vert \W\Vert_F^2}\right)\sum_{t=1}^Td_t\right]}\right)$ for constant step size. Here, $\mathcal{M}$ is the set of samples for which the feedback is not received till $T$, and $L$ is the upper bound on the loss for samples in $\mathcal{M}$. 
    \item When there is not boundedness assumption on the loss for samples in the set $\mathcal{M}$, the algorithm achieves a regret bound of $\mathcal{O}\left(\sqrt{\frac{2 K}{\gamma}\left[\frac{T}{2}+2\sum_{t=1}^Td_t+\vert \mathcal{M}\vert T\right]}\right)$ for constant step size.
    \item Constant step size requires the algorithm to know $T$ and $\sum_{t=1}^T d_t$ beforehand. In the case when  $T$ and $\sum_{t=1}^T d_t$ are not known, using doubling trick we proposed {\it Adaptive Delaytron} approach which achieves a regret bound of $\mathcal{O}(\sqrt{T+\sum_{t=1}^Td_t})$.
    \item We show experimentally that both the proposed algorithms (Delaytron and Adaptive Delaytron) can learn efficient classifiers using delayed bandit feedbacks for various datasets.
\end{enumerate}
 
\section{Proposed Approach: Delaytron}
Here, the goal is learn a function $g:\mathcal{X}\rightarrow [K]$ which given an example $\mathbf{x}\in \mathcal{X}\subseteq \mathbb{R}^d$, assigns a label in the set $[K]=\{1,\ldots,K\}$. A linear multi class classifier is modeled using a weight matrix $\W\in \mathbb{R}^{K\times d}$ as $g(\mathbf{x})=\arg\max_{j\in [K]}\;(\W\mathbf{x})_j$. To learn the weight matrix $W$, we use the training set $\{\mathbf{x}_1,y_1),\ldots,(\mathbf{x}_T,y_T)\}$ where $(\mathbf{x}_i,y_i)\in \mathcal{X}\times [K],\;\forall i\in [T]$. Performance of classifier $g$ is measured using 0-1 loss ($l_{0-1}(\W,(\mathbf{x}_i,y_i))=\mathbb{I}[g(\mathbf{x}_i)\neq y_i]$). In practice, convex surrogates of $l_{0-1}$ are used to learn the classifier. One such surrogate is multiclass hinge loss \cite{Crammer2002} described as 
\begin{align*}
    l_H(\W,(\mathbf{x}_i,y_i))=\max(0,1-(\W\mathbf{x}_i)_{y_i}+\max_{j\neq y_t}\;(\W\mathbf{x}_i)_j).
\end{align*}
The loss is 0 when $(\W\mathbf{x}_i)_{y_i}-\max_{j\neq y_t}\;(\W\mathbf{x}_i)_j\geq 1$ and $1-(\W\mathbf{x}_i)_{y_i}+\max_{j\neq y_t}\;(\W\mathbf{x}_i)_j$ otherwise.

At round $t$, the algorithm receives an example $\mathbf{x}_t$. It then finds $\hat{y}_t$ as $\hat{y}_t={\arg\max}_{j\in[K]}\;(\W_t\mathbf{x}_t)_j$. Then, it defines a distribution over the classes as $P_t(r)=(1-\gamma)\mathbb{I}[r=\hat{y}_t]+\frac{\gamma}{K},\;r\in[K]$. It samples $\tilde{y}_t$ using distribution $P^t$ and predicts $\tilde{y}_t$. The algorithm gets to know the bandit feedback ($\mathbb{I}[\tilde{y}_t=y_t]$) of predicting $\tilde{y}_t$ for round $t$ at the end of the $t+d_t-1$ round. In other words, the algorithms observes a delay of $d_t$ ($\geq 1$) to observe the feedback $\mathbb{I}[\tilde{y}_t=y_t]$. So, the feedback is available at the beginning of round $t+d_t$. Note that $y_t$ is the true label of the example $\mathbf{x}_t$. The goal of the algorithm is to minimize the number of mistakes $\hat{M}=\sum_{t=1}^T\mathbb{I}[\tilde{y}_t\neq y_t]$.

Let $\mathcal{S}_t$ denotes the set of feedbacks received at the beginning of round $t$. So, $r\in \mathcal{S}_t$ means that the feedback of predicting $\tilde{y}_r$ at round $r<t$ is received at round $t$. Since, the algorithm runs for $T$ rounds, all feedbacks for which $t+d_t>T$ are never received. These feedbacks are called missing samples and $\mathcal{M}$ denotes the set of missing samples. At round $t$, the algorithm considers all the feedbacks received at time $t$ and updates the parameters as $\W^{t+1}=\W^t+\eta\sum_{s\in \mathcal{S}_t}\tilde{\U}^s$, where $\eta>0$ is the step size. Also, for all $r\in [K],j\in [d]$, $\tilde{\U}^t_{r,j}$ is defined as
\begin{align*}
    \tilde{\U}^t_{r,j}=\mathbf{x}_t(j)\left[\frac{\mathbb{I}[\tilde{y}_t=y_t]\mathbb{I}[\tilde{y}_t=r]}{P_t(r)}-\mathbb{I}[\hat{y}_t=r]\right].
\end{align*}
Note that $\mathbb{E}_{\tilde{y}_t\sim P_t}[\tilde{\U}^t_{r,j}]=\U^t_{r,j}$ \cite{10.1145/1390156.1390212}, where matrix $\U^t$ is negative of the gradient of the loss $l_H(\W,(\mathbf{x}_t,y_t))$ evaluated at $\W^t$. The approach is described in details in Algorithm~1.

\begin{algorithm}[H]
    \caption{Delaytron}
    \label{alg:DBF}
    \textbf{Input}: $\gamma \in (0,0.5)$, Step size $\eta>0$\\
    \textbf{Initialize}: Set $\W^{1} = 0 \in \mathbb{R}^{K \times d}$ 
    \begin{algorithmic}[1] 
    \FOR{$t=1,\cdots,T$}
    \STATE Receive  $\mathbf{x}_{t} \in \mathbb{R}^{d}$. 
    \STATE Set $\hat{y}^{t} = {\arg \max}_{r\in[K]}\;(\W_{t}\mathbf{x}_{t})_r$
    \STATE Set $P_t(r) = (1-\gamma)\mathbb{I}[r=\hat{y}^{t}] + \frac{\gamma}{K},\; r\in[K]$
    \STATE Randomly sample $\tilde{y}_{t}$ according to $P_t$. 
    \STATE Predict $\tilde{y}_{t}$
    \STATE Obtain a set of delayed feedbacks $\mathbb{I}[\tilde{y}_s=y_s]$ for all $s\in \mathcal{S}_t$, where $\tilde{y}_s$ is the prediction made at round $s$ and $y_s$ is the corresponding ground truth.
    \FOR{$s\in S_t$}
    \STATE For all $r\in [K]$ and $j\in[d]$, compute $\tilde{\U}^s_{r,j}$ as below:\\ $\tilde{\U}^s_{r,j}=\mathbf{x}_t(j)\left[\frac{\mathbb{I}[\tilde{y}_s=y_s]\mathbb{I}[\tilde{y}_s=r]}{P_s(r)}-\mathbb{I}[\hat{y}_s=r]\right]$
    \ENDFOR
    \STATE Update: $\W^{t+1}$ = $\W^{t} + \eta \sum_{s\in \mathcal{S}_t}\tilde{\U}^s$
    \ENDFOR
    \end{algorithmic}
    \end{algorithm}

\section{Analysis}

We split sum of gradients in a single round and apply them one by one. For each $s\in \mathcal{S}_t$, let $\mathcal{S}_{t,s}=\{q\in \mathcal{S}_t\;:\;q<s\}$, which is the set of bandit feedback samples that the algorithm uses before the bandit feedback from round $s$ is used. Let $\mathcal{F}_t=\sigma\left(\{\tilde{y}_s\;|\;s+d_s\leq t\}\right)$,
which is generated from all the actions for which the feedback was received upto round $t$. 
The performance of the algorithm, is measured using the regret as follows.
\begin{align}
    \label{eq:regret-definition}
    \mathcal{R}(T)&=\mathbb{E}\left[\sum_{t=1}^T  l_H(\W_t,(\mathbf{x}_t,y_{t})) -\sum_{t=1}^Tl_H(\W^*,(\mathbf{x}_t,y_{t}))\right]
\end{align}
where $\W^*={\arg\min}_{\W}\sum_{t=1}^Tl_H(\W,(\mathbf{x}_t,y_{t})$.

We first show the regret bound for the examples other than missing samples. We use the following lemma to analyze the contribution of the "delay term" to the expected regret.
\begin{lemma}\cite{NEURIPS2019_ae2a2db4}
\label{lemma:count-lemma}
Let $d_t$ be the delay in the bandit feedback corresponding to round $t$. Let $\mathcal{S}_t$ be the set of feedbacks received at time $t$ and define $\mathcal{S}_{t,s}=\{r\in \mathcal{S}_t\;|\;r<s\}$ which is the set of feedback samples that the algorithm uses before the feedback from round $s$ is used. Let $\mathcal{M}$ be the set of missing samples, then
$\sum_{t=1}^T\sum_{s\in \mathcal{S}_t}\left[|\mathcal{S}_{t,s}|+\sum_{r=s}^{t-1}|\mathcal{S}_r|\right]\leq 2\sum_{t\notin \mathcal{M}}d_t$.
\end{lemma}
Proof of this lemma can be found in \cite{NEURIPS2019_ae2a2db4}. We now bound the regret $\mathbb{E}\left[\sum_{t\notin \mathcal{M}}\left\{l_H(\W^t,(x_t,y_t))-l_H(\W,(x_t,y_t))\right\}\right]$ with the help of following Theorem.
\begin{theorem}
Let $\{\mathbf{x}_t\}_{t=1}^T$ be the sequence of examples observed by Algorithm~1. Let $\mathbb{I}[\tilde{y}_t=y_t]$ be the bandit feedback corresponding to round $t$ which is observed only after a delay of $d_t$. Let $\Vert \mathbf{x}_t\Vert_2\leq R,\;\forall t\in[T]$. Then for any $\W\in \mathbb{R}^{K\times d}$, the regret achieved by Delaytron on examples for which the bandit feedback is received before $T$ is as follows.
\begin{align*}
 &\mathbb{E}\left[\sum_{t\notin \mathcal{M}}l_H(\W^t,(\x_t,y_t))-\sum_{t\notin \mathcal{M}}l_H(\W,(\x_t,y_t))\right]\\
 &\leq \frac{1}{2\eta}\Vert \W\Vert^2_F+\frac{\eta KR^2}{\gamma}\left[\frac{T}{2}+2\sum_{t\notin \mathcal{M}}d_t\right]
\end{align*}
\end{theorem}
\begin{proof}
Let $\W^{t,s}=\W^t + \eta \sum_{q\in \mathcal{S}_{t,s}} \tilde{U}^q$. Suppose $\mathcal{S}_t$ is nonempty and let $s'=\max \mathcal{S}_t$ be the last index in the set $\mathcal{S}_t$. Thus, we have
\begin{align*}
   \Vert \W^{t+1}-\W\Vert^2_F&=\Vert \W^{t,s'}-\W\Vert^2_F +2\eta\langle \W^{t,s'}-\W,\tilde{U}^{s'}\rangle \\
   &+ \eta^2\Vert \tilde{U}^{s'}\Vert_F^2
\end{align*}
Repeatedly unrolling the first term on the RHS gives us the following
\begin{align*}
    \Vert \W^{t+1}-\W\Vert^2_F&=\Vert \W^{t}-\W\Vert^2_F + \eta^2\sum_{s\in \mathcal{S}_t}\Vert \tilde{U}^{s}\Vert_F^2\\
    &+2\eta\sum_{s\in \mathcal{S}_t}\langle \W^{t,s}-\W,\tilde{U}^{s}\rangle 
\end{align*}
Taking expectation on both sides, we get
\begin{align*}
    &\mathbb{E}_{\mathcal{F}_t}[\Vert \W^{t+1}-\W\Vert^2_F]=\mathbb{E}_{\mathcal{F}_t}[\Vert \W^{t}-\W\Vert^2_F]\\
    &\;\;\;\;+2\eta\mathbb{E}_{\mathcal{F}_t}[\sum_{s\in \mathcal{S}_t}\langle \W^{t,s}-\W,\tilde{U}^{s}\rangle] + \eta^2\mathbb{E}_{\mathcal{F}_t}[\sum_{s\in \mathcal{S}_t}\Vert \tilde{U}^{s}\Vert_F^2]\\
    &=\mathbb{E}_{\mathcal{F}_t}[\Vert \W^{t}-\W\Vert^2_F]+2\eta\mathbb{E}_{\mathcal{F}_t}[\sum_{s\in \mathcal{S}_t}\langle \W^{t,s}-\W,U^{s}\rangle]\\
    &\;\;\;\;+ \eta^2\mathbb{E}_{\mathcal{F}_t}[\sum_{s\in \mathcal{S}_t}\Vert \tilde{U}^{s}\Vert_F^2]
    \end{align*}
    Where we used the fact that $\mathbb{E}_{\mathcal{F}_t}[\langle \W^{t,s}-\W,\tilde{U}^{s}\rangle]=\mathbb{E}_{\mathcal{F}_t}[\langle \W^{t,s}-\W,{U}^{s}\rangle]$. Now,
using the convexity property of the loss function $l_H(\W,(\mathbf{x},y))$, we know that
\begin{align*}
    &\langle \W^{t,s}-\W,U^{s}\rangle =\langle \W^{s}-\W,U^{s}\rangle+\langle \W^{t,s}-\W^s,U^{s}\rangle\\
    &\leq l_H(\W,(x_s,y_s))-l_H(\W^s,(x_s,y_s))+\langle \W^{t,s}-\W^s,U^{s}\rangle.
\end{align*}
Using this, we get 
\begin{align*}
&\mathbb{E}_{\mathcal{F}_t}\left[\sum_{s\in \mathcal{S}_t}l_H(\W^s,(x_s,y_s))\right]-\sum_{s\in \mathcal{S}_t}l_H(\W,(x_s,y_s))]\\
&\leq \frac{1}{2\eta}\mathbb{E}_{\mathcal{F}_t}\Bigg[\Vert \W^{t}-\W\Vert^2_F-\Vert \W^{t+1}-\W\Vert^2_F\\
&\;\;\;\;+\eta^2\sum_{s\in \mathcal{S}_t}\Vert \tilde{U}^{s}\Vert_F^2\Bigg]+\mathbb{E}_{\mathcal{F}_t}\left[\sum_{s\in \mathcal{S}_t}\langle \W^{t,s}-\W^s,U^{s}\rangle\right]
    \end{align*}
Summing from $t=1$ to $T$, we get the following.
\begin{align*}
&\mathbb{E}\left[\sum_{t=1}^T\sum_{s\in \mathcal{S}_t}l_H(\W^s,(x_s,y_s))\right]-\sum_{t=1}^T\sum_{s\in \mathcal{S}_t}l_H(\W,(x_s,y_s))\\
&\leq \frac{1}{2\eta}\mathbb{E}\left[\sum_{t=1}^T\left[\Vert \W^{t}-\W\Vert^2_F-\Vert \W^{t+1}-\W\Vert^2_F\right]\right]\\
&+\frac{\eta}{2}\mathbb{E}\left[\sum_{t=1}^T\sum_{s\in \mathcal{S}_t}\Vert \tilde{U}^{s}\Vert_F^2\right]+\mathbb{E}\left[\sum_{t=1}^T\sum_{s\in \mathcal{S}_t}\langle \W^{t,s}-\W^s,U^{s}\rangle\right]
\end{align*}
Using the fact that $\W^1=\mathbf{0}\in \mathbb{R}^{K\times d}$, we get \begin{align*}
    &\sum_{t=1}^T[\Vert \W^{t}-\W\Vert^2_F-\Vert \W^{t+1}-\W\Vert^2_F]= \Vert \W^1-\W\Vert_F^2\\
    &\;\;\;\;-\Vert \W^{T+1}-\W\Vert_F^2\leq \Vert \W^1-\W\Vert_F^2=\Vert \W\Vert_F^2.
    \end{align*}
    This gives, 
\begin{align*}
\nonumber &\mathbb{E}\left[\sum_{t=1}^T\sum_{s\in \mathcal{S}_t}l_H(\W^s,(x_s,y_s))\right]-\sum_{t=1}^T\sum_{s\in \mathcal{S}_t}l_H(\W,(x_s,y_s))\\
&\;\;\leq \frac{1}{2\eta}\Vert \W\Vert^2_F+\frac{\eta}{2}\sum_{t=1}^T\sum_{s\in \mathcal{S}_t}\mathbb{E}_{P_s}\left[\Vert \tilde{U}^{s}\Vert_F^2\right]\\
    &\;\;\;\;+\mathbb{E}\left[\sum_{t=1}^T\sum_{s\in \mathcal{S}_t}\langle \W^{t,s}-\W^s,U^{s}\rangle\right].
\end{align*}
Using Lemma~5 in \cite{10.1145/1390156.1390212}, we know that $\mathbb{E}_{P_s}\left[\Vert \tilde{U}^s \Vert_F^2\right] \leq  2\Vert \mathbf{x}_s\Vert^2\left(\frac{K}{\gamma}\mathbb{I}[y_s\neq \hat{y}_s]+\gamma \mathbb{I}[y_s= \hat{y}_s\right)$. Which can be further upper bounded as $\mathbb{E}_{P_s}\left[\Vert \tilde{U}^s \Vert_F^2\right] \leq \frac{K R^2}{\gamma}$, where we used the fact that $R^2=\max_{s\in [T]}\;\Vert \mathbf{x}\Vert^2$. Using this, the bound becomes
\begin{align}
    \nonumber &\mathbb{E}\left[\sum_{t=1}^T\sum_{s\in \mathcal{S}_t}l_H(\W^s,(x_s,y_s))\right]-\sum_{t=1}^T\sum_{s\in \mathcal{S}_t}l_H(\W,(x_s,y_s))\\
    &\leq  \frac{1}{2\eta}\Vert \W\Vert^2_F+\frac{\eta KTR^2}{2\gamma}+\mathbb{E}\left[\sum_{t=1}^T\sum_{s\in \mathcal{S}_t}\langle \W^{t,s}-\W^s,U^{s}\rangle\right].\label{eq:regret-temp}
\end{align}
Each summand $\langle \W^{t,s}-\W^s,U^{s}\rangle$ contributes loss proportional to the distance between the matrix $\W^s$ when update $\tilde{U}^s$ is generated and the matrix $\W^{t,s}$ when $\tilde{U}^s$ is applied. This distance is created by the other updates that are applied in between. Using Cauchy-Schwartz inequality, we first bound the delay terms as follows.
\begin{align}
    \nonumber \sum_{t=1}^T\sum_{s\in \mathcal{S}_t}\langle \W^{t,s}-\W^s,U^{s}\rangle&\leq \sum_{t=1}^T\sum_{s\in \mathcal{S}_t}\Vert \W^{t,s}-\W^s\Vert_F \Vert U^{s}\Vert_F\\
    &\leq R\sum_{t=1}^T\sum_{s\in \mathcal{S}_t}\Vert \W^{t,s}-\W^s\Vert_F
    \label{eq:bound-sumProd}
\end{align}
Consider a single term $\Vert \W^{t,s}-\W^{s}\Vert_F$ for fixed $t\in[T]$ and $s\in \mathcal{S}_t$. The difference $\W^{t,s}-\W^{s}$ is roughly the sum of updates received between round $s$ and when we apply the update from round $s$ in round $t$. Using triangle inequality, we have
\begin{align*}
    &\Vert \W^{t,s}-\W^s\Vert_F\leq \Vert \W^{t,s}-\W^t\Vert_F+\Vert \W^{t}-\W^s\Vert_F\\
    &= \Vert \W^{t,s}-\W^t\Vert_F+\Vert \sum_{r=s}^{t-1}(\W^{r+1}-\W^r)\Vert_F\\
    &\leq \Vert \W^{t,s}-\W^t\Vert_F+\sum_{r=s}^{t-1}\Vert \W^{r+1}-\W^r\Vert_F\\
    &\leq \eta \sum_{p\in \mathcal{S}_{t,s}}\Vert \tilde{U}^p \Vert_F+\eta\sum_{r=s}^{t-1}\sum_{q\in \mathcal{S}_r}\Vert \tilde{U}^q\Vert_F.
    \end{align*}
    Taking expectation on both sides, we get following bound.
    \begin{align}
   \nonumber  \mathbb{E}[\Vert \W^{t,s}-\W^s\Vert_F] &\leq \mathbb{E}\left[\sum_{p\in \mathcal{S}_{t,s}}\Vert \tilde{U}^p \Vert_F+\sum_{r=s}^{t-1}\sum_{q\in \mathcal{S}_r}\Vert \tilde{U}^q\Vert_F\right]
   \end{align}
   Using the concavity property of square root function, we observe that $\mathbb{E}_{P_s}[\Vert \tilde{U}^s\Vert_F]=\mathbb{E}_{P_s}[\sqrt{\Vert \tilde{U}^s\Vert_F^2}]\leq \sqrt{\mathbb{E}_{P_s}[\Vert \tilde{U}^s\Vert_F^2]}\leq \frac{\sqrt{K}R}{\sqrt{\gamma}}$. Thus,
   \begin{align}
     \mathbb{E}[\Vert \W^{t,s}-\W^s\Vert_F] & \leq \sqrt{\frac{KR^2}{\gamma}}\left(|\mathcal{S}_{t,s}|+\sum_{r=s}^{t-1}|\mathcal{S}_r|\right)
    \label{eq:bound-dist}
\end{align}
Using eq.(\ref{eq:bound-dist}) and (\ref{eq:bound-sumProd}) in eq.(\ref{eq:regret-temp}), we get the following.
\begin{align}
\nonumber &\mathbb{E}\left[\sum_{t=1}^T\sum_{s\in \mathcal{S}_t}l_H(\W^s,(x_s,y_s))\right]-\sum_{t=1}^T\sum_{s\in \mathcal{S}_t}l_H(\W,(x_s,y_s))\\
&\leq \frac{1}{2\eta}\Vert \W\Vert^2_F+\frac{KR^2}{\gamma}\left[\frac{\eta T}{2}+\sum_{t=1}^T\sum_{s\in \mathcal{S}_t}\left\{|\mathcal{S}_{t,s}|+\sum_{r=s}^{t-1}|\mathcal{S}_r|\right\}\right]
    \label{eq:regretbound-temp1}
\end{align}
Using Lemma~\ref{lemma:count-lemma}, we have $\sum_{t=1}^T\sum_{s\in \mathcal{S}_t}\left[|\mathcal{S}_{t,s}|+\sum_{r=s}^{t-1}|\mathcal{S}_r|\right]\leq 2\sum_{t\notin \mathcal{M}} d_t$. Using this in eq.(\ref{eq:regretbound-temp1}), we get
\begin{align*}
&\mathbb{E}\left[\sum_{t=1}^T\sum_{s\in \mathcal{S}_t}l_H(\W^s,(\mathbf{x}_s,y_s))\right]-\sum_{t=1}^T\sum_{s\in \mathcal{S}_t}l_H(\W,(\mathbf{x}_s,y_s))\\
&\leq \frac{1}{2\eta}\Vert \W\Vert^2_F+\frac{\eta KR^2}{\gamma}\left[\frac{T}{2}+2\sum_{t\notin \mathcal{M}}d_t\right]
\end{align*}
But, $\sum_{t=1}^T\sum_{s\in \mathcal{S}_t}\left\{l_H(\W^s,(x_s,y_s))-l_H(\W,(x_s,y_s)\right\}=\sum_{t\notin \mathcal{M}}\left\{l_H(\W^t,(x_t,y_t))-l_H(\W,(x_t,y_t)\right\}$. Thus, 
\begin{align}
 \nonumber &\mathbb{E}\left[\sum_{t\notin \mathcal{M}}l_H(\W^t,(x_t,y_t))-\sum_{t\notin \mathcal{M}}l_H(\W,(x_t,y_t))\right]\\
 &\leq \frac{1}{2\eta}\Vert \W\Vert^2_F+\frac{\eta KR^2}{\gamma}\left[\frac{T}{2}+2\sum_{t\notin \mathcal{M}}d_t\right].\label{eq:Regret-Unmissed}
\end{align}
\end{proof}
\subsection{Case 1: No Delay} When there is no delay, set $\mathcal{M}$ is empty and $d_t=1,\forall t\in [T]$. Thus, $\sum_{t\notin \mathcal{M}}d_t=\sum_{t=1}^Td_t=T$. In that case the regret bound becomes 
\begin{align*}
 &\mathbb{E}\left[\sum_{t=1}^T l_H(\W^t,(\x_t,y_t))-\sum_{t=1}^Tl_H(\W,(\x_t,y_t))\right]\\
 &\leq \frac{1}{2\eta}\Vert \W\Vert^2_F+\frac{5\eta KTR^2}{2\gamma}.
\end{align*}
Using $\eta=\frac{\Vert \W\Vert_F}{\sqrt{\frac{KTR^2}{\gamma}}}$, the regret bound becomes
\begin{align*}
 &\mathbb{E}\left[\sum_{t=1}^T l_H(\W^t,(\x_t,y_t))-\sum_{t=1}^Tl_H(\W,(\x_t,y_t))\right]\\
 &\leq 6R\Vert \W\Vert_F\sqrt{\frac{KT}{\gamma}},\;\forall \;\W\in \mathbb{R}^{K\times d}.
\end{align*}
We now derive the complete regret bound for two different cases. These cases are based on whether the losses for the missing samples are bounded.

\subsection{Case 2: Bounded Loss for Missing Samples}
Here, we assume that the loss $l_H(\W,(\mathbf{x}_t,y_t))$ is upper bounded for all the missing samples. Which means, $l_H(\W,(\mathbf{x}_t,y_t))\leq L, \;\forall t\in \mathcal{M}$. This seems a fair assumption as the bandit feedback for such examples is never revealed. 
\begin{theorem}
Let $\{\mathbf{x}_t\}_{t=1}^T$ be the sequence of examples observed by Algorithm~1. Let $\mathbb{I}[\tilde{y}_t=y_t]$ be the bandit feedback corresponding to round $t$ which is observed only after a delay of $d_t$. Let $\Vert \mathbf{x}_t\Vert_2\leq R,\;\forall t\in[T]$ and $\mathcal{M}=\{t\in[T]\;|\;t+d_t>T\}$. Let $l_H(\W,(\mathbf{x}_t,y_t))\leq L,\;\forall t\in \mathcal{M}, \forall W\in \mathbb{R}^{K\times d}$. Then for any $\W\in \mathbb{R}^{K\times d}$, the regret achieved by Delaytron is as follows.
\begin{align*}
&\mathcal{R}(T)\leq  R\Vert \W\Vert_F\sqrt{\frac{2 K}{\gamma}\left[\frac{T}{2}+2\sum_{t\notin \mathcal{M}}d_t\right]}+L\vert \mathcal{M}\vert
\end{align*}
\end{theorem}
\begin{proof}
\begin{align}
\nonumber &\mathbb{E}\left[\sum_{t\in \mathcal{M}}l_H(\W^t,(\x_t,y_t))-\sum_{t\in \mathcal{M}}l_H(\W,(\x_t,y_t))\right]\\
&\leq \mathbb{E}\left[\sum_{t\in \mathcal{M}}l_H(\W^t,(\x_t,y_t))\right]\leq L\vert\mathcal{M}\vert\label{eq:Regret-Missing-Samples}
\end{align}
Using Theorem~1 and (\ref{eq:Regret-Missing-Samples}), we get, the regret as follows.
\begin{align*}
\mathcal{R}(T)&\leq \frac{1}{2\eta}\Vert \W\Vert^2_F+\frac{\eta KR^2}{\gamma}\left[\frac{T}{2}+2\sum_{t\notin \mathcal{M}}d_t\right]+L\vert\mathcal{M}\vert
\end{align*}
Using $\eta = \frac{\Vert \W\Vert_F}{\sqrt{\frac{2 KR^2}{\gamma}\left[\frac{T}{2}+2\sum_{t\notin \mathcal{M}}d_t\right]}}$, we get the following bound.
\begin{align*}
\mathcal{R}(T)\leq  R\Vert \W\Vert_F\sqrt{\frac{2 K}{\gamma}\left[\frac{T}{2}+2\sum_{t\notin \mathcal{M}}d_t\right]}+L\vert \mathcal{M}\vert
\end{align*}
\qed
\end{proof}
Our bound does not account for the delays of the missing samples, and it only depends on $\sum_{t\notin \mathcal{M}}d_t$. Moreover, counting delays that go beyond $T$ is redundant.
It is worth noting that $m=\vert \mathcal{M}\vert$ contribute at least $m(m+1)/2$ to the sum of delays $\sum_{t=1}^Td_t$. It happens in the following scenario. Feedback of round $T$ is delayed by one, feedback of round $(T-1)$ is delayed by two, and so on $m$ times. We use this idea and get regret bound independent of $\mathcal{M}$ in the following corollary.
\begin{corollary}
Let $\{\mathbf{x}_t\}_{t=1}^T$ be the sequence of examples observed by Algorithm~1. Let $\mathbb{I}[\tilde{y}_t=y_t]$ be the bandit feedback corresponding to round $t$ which is observed only after a delay of $d_t$. Let $\Vert \mathbf{x}_t\Vert_2\leq R,\;\forall t\in[T]$ and $\mathcal{M}=\{t\in[T]\;|\;t+d_t>T\}$. Let $l_H(\W,(\mathbf{x}_t,y_t))\leq L,\;\forall t\in \mathcal{M}, \forall W\in \mathbb{R}^{K\times d}$. Then for any $W\in \mathbb{R}^{K\times d}$, the regret of Delaytron is upper bounded as follows.
\begin{align*}
\mathcal{R}(T)\leq  R\Vert \W\Vert_F\sqrt{\frac{2 K}{\gamma}\left[\frac{T}{2}+\left(2+\frac{L^2}{R^2\Vert \W\Vert_F^2}\right)\sum_{t=1}^Td_t\right]}
\end{align*}
\end{corollary}
\begin{proof}
Thus,
\begin{align*}
  &\mathcal{R}(T)\leq R\Vert \W\Vert^2 \sqrt{\frac{2 K}{\gamma}\left[\frac{T}{2}+2\sum_{t\notin \mathcal{M}}d_t\right]}+L\vert \mathcal{M}\vert\\
  &\leq \frac{R\Vert \W\Vert_F}{2}\sqrt{\frac{2 K}{\gamma}\left[\frac{T}{2}+2\sum_{t\notin \mathcal{M}}d_t\right]}+\frac{1}{2}\sqrt{\frac{KL^2m(m+1)}{\gamma}}\\
  &\leq R\Vert \W\Vert_F\sqrt{\frac{2 K}{\gamma}\left[\frac{T}{2}+2\sum_{t\notin \mathcal{M}}d_t+\frac{L^2m(m+1)}{2R^2\Vert \W\Vert_F^2}\right]}\\
  & \leq R\Vert \W\Vert_F\sqrt{\frac{2 K}{\gamma}\left[\frac{T}{2}+\left(2+\frac{L^2}{R^2\Vert \W\Vert_F^2}\right)\sum_{t=1}^Td_t\right]}
\end{align*}
Where the second inequality follows from the concavity property of $f(x)=\sqrt{x}$. 

\end{proof}
The regret bound showed above is independent of $\mathcal{M}$. However, it depends on $T$ and $\sum_{t=1}^Td_t$. 
\subsection{Without Boundedness Condition on the Loss for Missing Samples }
Here, we do not assume that the loss for the missing samples is upper bounded by a fixed quantity. Now, let us find regret on missing samples.
\begin{theorem}
Let $\{\mathbf{x}_t\}_{t=1}^T$ be the sequence of examples observed by Algorithm~1. Let $\mathbb{I}[\tilde{y}_t=y_t]$ be the bandit feedback corresponding to round $t$ which is observed only after a delay of $d_t$. Let $\Vert \mathbf{x}_t\Vert_2\leq R,\;\forall t\in[T]$ and $\mathcal{M}=\{t\in[T]\;|\;t+d_t>T\}$. Then for any $W\in \mathbb{R}^{K\times d}$, the regret achieved by Delaytron is as follows.
\begin{align*}
&\mathcal{R}(T)\leq  R\Vert \W\Vert_F\left[\vert \mathcal{M}\vert+\sqrt{\frac{2 K}{\gamma}\left[\frac{T}{2}+2\sum_{t\notin \mathcal{M}}d_t+\vert \mathcal{M}\vert T\right]}\right]
\end{align*}
\end{theorem}
\begin{proof}
\begin{align*}
\nonumber &\mathbb{E}\left[\sum_{t\in \mathcal{M}}\left\{l_H(\W^t,(\x_t,y_t))-l_H(\W,(\x_t,y_t))\right\}\right]\\
&\leq \mathbb{E}\left[\sum_{t\in \mathcal{M}}\langle \nabla l_H(\W^t,(\x_t,y_t)), \W^t-\W^*\rangle\right]\\
&\leq \mathbb{E}\left[\sum_{t\in \mathcal{M}}  \Vert \nabla l_H(\W^t,(\x_t,y_t)) \Vert_F. \Vert \W^t-\W^*\Vert_F\right]
\end{align*}
But, $\Vert \nabla l_H(\W^t,(\mathbf{x}_t,y_t))\Vert_F\leq \Vert\mathbf{x}_t\Vert_2\leq R$. Also, $\Vert \W^t-\W\Vert_F\leq \Vert \W^t\Vert_F+\Vert \W\Vert_F$. But, $\W^t=\eta\sum_{r=1}^{t-1}\sum_{s\in \mathcal{S}_r} \tilde{\U}^s$. Thus, $\Vert \W^t\Vert_F\leq \eta \sum_{r=1}^{t-1}\sum_{s\in \mathcal{S}_r}\Vert \tilde{\U}^s\Vert_F$. Taking expectation, we get, $\mathbb{E}[\Vert \W^t\Vert_F]\leq \eta \sum_{r=1}^{t-1}\sum_{s\in \mathcal{S}_r}\mathbb{E}[\Vert \tilde{\U}^s\Vert_F]\leq \eta R\sqrt{\frac{K}{\gamma}} \sum_{r=1}^{t-1}\vert \mathcal{S}_r\vert$. $\sum_{r=1}^{t-1}\vert \mathcal{S}_r\vert$ is the number of feedbacks received till time $t-1$ which is not more than $t-1$. Thus, $\mathbb{E}[\Vert \W^t\Vert_F]\leq  \frac{\eta K R}{\gamma}(t-1)\leq  \frac{\eta KT R}{\gamma}$. So, we get
\begin{align}
 \nonumber &\mathbb{E}\left[\sum_{t\in \mathcal{M}}l_H(\W^t,(\x_t,y_t))-\sum_{t\in \mathcal{M}}l_H(\W,(\x_t,y_t))\right]\\
 &\leq \vert \mathcal{M}\vert R(\Vert \W\Vert_F+\frac{\eta KT R}{\gamma})\label{eq:Regret-Missing-Samples1}
\end{align}
Using Theorem~1 and eq.(\ref{eq:Regret-Missing-Samples1}), we get, the overall regret as follows.
\begin{align*}
&\mathcal{R}(T)\leq \vert \mathcal{M}\vert R\Vert \W\Vert_F+
\frac{1}{2\eta}\Vert \W\Vert^2_F\\
&+\frac{\eta KR^2}{\gamma}\left[\frac{T}{2}+2\sum_{t\notin \mathcal{M}}d_t+\vert \mathcal{M}\vert T\right]
\end{align*}
Using $\eta = \frac{\Vert \W\Vert_F}{\sqrt{\frac{2 KR^2}{\gamma}\left[\frac{T}{2}+2\sum_{t\notin \mathcal{M}}d_t+\vert \mathcal{M}\vert T\right]}}$, we get the following bound.
\begin{align*}
\mathcal{R}(T)\leq  R\Vert \W\Vert_F\left[\vert \mathcal{M}\vert+\sqrt{\frac{2 K}{\gamma}\left[\frac{T}{2}+2\sum_{t\notin \mathcal{M}}d_t+\vert \mathcal{M}\vert T\right]}\right]
\end{align*}
\end{proof}
We get $\mathcal{O}\left(\sqrt{\frac{T}{2}+2\sum_{t\notin \mathcal{M}}d_t+\vert \mathcal{M}\vert T}\right)$ regret bound in case when we don't have any boundedness assumption on the loss incurred on missing samples. Corollary below give further refined regret bound.
\begin{corollary}
Let $\{\mathbf{x}_t\}_{t=1}^T$ be the sequence of examples observed by Algorithm~1. Let $\mathbb{I}[\tilde{y}_t=y_t]$ be the bandit feedback corresponding to round $t$ which is observed only after a delay of $d_t$. Let $\Vert \mathbf{x}_t\Vert_2\leq R,\;\forall t\in[T]$ and $\mathcal{M}=\{t\in[T]\;|\;t+d_t>T\}$. Then for any $W\in \mathbb{R}^{K\times d}$, the regret achieved by Delaytron is as follows.
\begin{align*}
\mathcal{R}(T)\leq  \mathcal{O}\left(R\Vert W\Vert_F\sqrt{\frac{2 K}{\gamma}\left[\frac{T}{2}+2\sum_{t=1}^Td_t+\vert \mathcal{M}\vert T\right]}\right)
\end{align*}
\end{corollary}
\begin{proof}
We see that 
\begin{align*}
&\sqrt{\frac{2 K}{\gamma}\left[\frac{T}{2}+2\sum_{t\notin \mathcal{M}}d_t+\vert \mathcal{M}\vert T\right]}+\vert \mathcal{M}\vert\\
&\leq \frac{1}{2}\sqrt{\frac{2 K}{\gamma}\left[\frac{T}{2}+2\sum_{t\notin \mathcal{M}}d_t+\vert \mathcal{M}\vert T\right]}+\frac{1}{2}\sqrt{\frac{Km(m+1)}{\gamma}}\\
&\leq \sqrt{\frac{2 K}{\gamma}\left[\frac{T}{2}+2\sum_{t\notin \mathcal{M}}d_t+m(m+1)/2+\vert \mathcal{M}\vert T\right]}\\
   & \leq \sqrt{\frac{2 K}{\gamma}\left[\frac{T}{2}+2\sum_{t=1}^Td_t+\vert \mathcal{M}\vert T\right]}.
   \end{align*}
Where the second inequality follows from the concavity property of $f(x)=\sqrt{x}$. Thus,
\begin{align*}
&\mathcal{R}(T)\leq  R\Vert \W\Vert_F\left[\sqrt{\frac{2 K}{\gamma}\left[\frac{T}{2}+2\sum_{t\notin \mathcal{M}}d_t+\vert \mathcal{M}\vert T\right]}+\vert \mathcal{M}\vert\right]
\\ 
&\leq R\Vert \W\Vert_F\sqrt{\frac{2 K}{\gamma}\left[\frac{T}{2}+2\sum_{t=1}^Td_t+\vert \mathcal{M}\vert T\right]}
\end{align*}
\end{proof}

\section{Adaptive Delaytron for Unknown $T$ and $\sum_{t=1}^T d_t$}
The fixed step size $\eta$ used in Delaytron requires the knowledge of $T$ and $\sum_{t\notin\mathcal{M}}d_t$. We can make it independent of $T$ and $\sum_{t\notin\mathcal{M}}d_t$ by using doubling trick. However, with delays, the standard doubling trick \cite{10.5555/2371238} does not work. We use the doubling trick proposed in \cite{NEURIPS2019_ae2a2db4}, where $T$ and $\sum_{t\notin \mathcal{M}}d_t$ are unknown. We define $m_t$ as the number of samples for which bandit feedback is missing till round $t$. The idea is to start a new epoch every time $\sum_{s=1}^tm_s$ (that tracks $\sum_{s=1}^td_s$) doubles. We define the $e$-th epoch as 
\begin{align}
    \mathcal{T}_e=\{t\;\vert \; 2^{e-1} \leq \sum_{s=1}^t m_s < 2^e\}.
\end{align}
$\mathcal{T}_e$ is the set of consecutive rounds in which the sum of delays is within a given interval. During the $e$-th epoch, the algorithm uses the step size $\eta_e=\sqrt{\frac{1}{2^e}}$. The resulting approach is called Adaptive Delaytron and is discussed in details in Algorithm~\ref{algo:algo2}.

\begin{algorithm}[h]
    \caption{Adaptive Delaytron for Unknown $T$ and $\sum_{t=1}^T d_t$}
    \label{algo:algo2}
    \textbf{Input}: $\gamma,\beta \in (0,1)$, Step size $\eta>0$\\
    \textbf{Initialize}: Set $W^{1} = 0 \in \mathbb{R}^{K \times d}$. Set the epoch index $e=0$ and $\eta_0=1$. 
    \begin{algorithmic}[1] 
    \FOR{$t=1,\cdots,T$}
    \STATE Receive  $\mathbf{x}_{t} \in \mathbb{R}^{d}$. 
    \STATE Set $\hat{y}^{t} = {\arg \max}_{r\in[K]}\;(W_{t}\mathbf{x}_{t})_r$
    \STATE Set $P_t(r) = (1-\gamma)\mathbb{I}[r=\hat{y}^{t}] + \frac{\gamma}{K},\; r\in[K]$
    \STATE Randomly sample $\tilde{y}_{t}$ according to $P_t$. 
    \STATE Predict $\tilde{y}_{t}$
    \STATE Obtain a set of delayed feedbacks $\mathbb{I}[\tilde{y}_s=y_s]$ for all $s\in \mathcal{S}_t$, where $\tilde{y}_s$ is the prediction made at round $s$ and $y_s$ is the corresponding ground truth.
    \STATE Update the number of missing samples so far 
    \begin{align*}
        m_t=t-\sum_{s=1}^t\vert \mathcal{S}_s\vert
    \end{align*}
    \STATE If $\sum_{s=1}^t m_t\geq 2^e$, then update $e=e+1$ 
    \FOR{$s\in S_t$}
    \STATE For all $r\in [K]$ and $j\in[d]$, compute $\tilde{U}^s_{r,j}$ as below:\\ $\tilde{U}^s_{r,j}=\mathbf{x}_t(j)\left[\frac{\mathbb{I}[\tilde{y}_s=y_s]\mathbb{I}[\tilde{y}_s=r]}{P_s(r)}-\mathbb{I}[\hat{y}_s=r]\right]$
    \ENDFOR
    \STATE Update: $W^{t+1}$ = $W^{t} + \eta_e \sum_{s\in \mathcal{S}_t}\tilde{U}^s$ where $\eta_e=\sqrt{\frac{1}{2^e}}$
    \ENDFOR
    \end{algorithmic}
    \end{algorithm}
\begin{theorem}
Let $\{\mathbf{x}_t\}_{t=1}^T$ be the sequence of examples observed by Algorithm~1. Let $\mathbb{I}[\tilde{y}_t=y_t]$ be the bandit feedback corresponding to round $t$ which is observed only after a delay of $d_t$. Let $\Vert \mathbf{x}_t\Vert_2\leq R,\;\forall t\in[T]$ and $\mathcal{M}=\{t\in[T]\;|\;t+d_t>T\}$. Let $l_H(W,(\mathbf{x}_t,y_t))\leq L,\;\forall t\in \mathcal{M}, \forall W\in \mathbb{R}^{K\times d}$. Then for any $W\in \mathbb{R}^{K\times d}$, the regret achieved by Adaptive Delaytron is as follows.
\begin{align*}
\mathcal{R}(T)\leq  10\left[\frac{1}{2}\Vert W\Vert_F^2+\frac{KR^2}{\gamma}+L\right]\sqrt{T+\sum_{t=1}^Td_t}
\end{align*}
\end{theorem}
\begin{proof}
Define $\mathcal{M}_e$ as the set of bandit feedbacks in epoch $e$ that are not received withing epoch $e$. Denote by $T_e=\max \mathcal{T}_e$ as the last round in $\mathcal{T}_e$. Note that $\mathcal{T}_e$ is also set of consecutive rounds from $T_{e-1}+1$ to $T_e$. Every round $t\notin \mathcal{T}_e$ such that $t\in \mathcal{M}_e$ contributes exactly $d_t$ to $\sum_{s=T_{e-1}+1}^{T_e}m_s$, since the $t$-th feedback is missing for $d_t$ rounds sometime between $T_{e-1}+1$ and $T_e$. Therefore,
\begin{align*}
    \sum_{t\in \mathcal{T}_e,t\notin \mathcal{M}_e}d_t \leq \sum_{s=T_{e-1}+1}^{T_e}m_s \leq 2^{e-1}.
\end{align*}
Where the last inequality uses the fact that if $\sum_{s=T_{e-1}+1}^{T_e}m_s > 2^{e-1}\geq 2^{e-1}+2^{e-1}=2^e$, then epoch $e+1$ should have been already started. We apply Theorem~1 separately on each epoch, which gives upper bound on the regret of epoch $e$ as follows.
\begin{align*}
    R_e &:= \mathbb{E}\left[\sum_{t\in \mathcal{T}_e} l_H(\W^t,(\mathbf{x}_t,y_t))- \sum_{t\in \mathcal{T}_e} l_H(\W,(\mathbf{x}_t,y_t))\right]\\
    &\leq  \frac{1}{2\eta_e}\Vert \W\Vert^2_F+\frac{\eta_e KR^2}{\gamma}\left[\frac{T_e}{2}+2\sum_{t\in \mathcal{T}_e,t\notin \mathcal{M}_e}d_t\right]+L\vert \mathcal{M}_e\vert
\end{align*}

Now we want to find the largest set $\mathcal{M}_e$ such that $\sum_{s=T_{e-1}+1}^{T_e} m_s \leq 2^{e-1}$ is still possible. The cheapest way to that is when the feedback from round $T_e$ is delayed by 1 (contributing 1 to $\sum_{s=T_{e-1}+1}^{T_e} m_s$), feedback from round $T_{e}-1$ is delayed by 2 (contributing 2 to $\sum_{s=T_{e-1}+1}^{T_e} m_s$) and so on. This gives $\sum_{i=1}^{\vert \mathcal{M}_e\vert}i=\frac{\vert \mathcal{M}_e\vert (\vert \mathcal{M}_e\vert +1)}{2}\leq 2^{e-1}$. This can happen if $\vert \mathcal{M}_e\vert \leq 2^{\frac{e}{2}}$. By choosing $\eta_e=2^{-\frac{e}{2}}$, we obtain
\begin{align*}
    R_e & \leq  \frac{1}{2}2^{\frac{e}{2}}\Vert \W\Vert^2_F+\frac{2^{-\frac{e}{2}} KR^2}{\gamma}\left[\frac{T_e}{2}+2^e\right]+L2^{\frac{e}{2}}\\
    &\leq 2^{\frac{e}{2}}\left[\frac{1}{2}\Vert \W\Vert_F^2+\frac{KR^2}{\gamma}+L\right]+\frac{2KR^2T_e}{\gamma}2^{-\frac{e}{2}}
\end{align*}
Let the last epoch be denoted as $E$, then we can conclude that
\begin{align*}
    &\mathbb{E}[\mathcal{R}(T)]=\sum_{e=1}^ER_e\\
    & \leq \sum_{e=1}^E \left\{2^{\frac{e}{2}}\left[\frac{1}{2}\Vert \W\Vert_F^2+\frac{KR^2}{\gamma}+L\right]+\frac{2KR^2T_e}{\gamma}2^{-\frac{e}{2}}\right\}\\
    &=\frac{2^{E/2}-1}{\sqrt{2}-1}\left[\frac{1}{2}\Vert \W\Vert_F^2+\frac{KR^2}{\gamma}+L\right]+\frac{2KR^2}{\gamma}\sum_{e=1}^ET_e2^{-\frac{e}{2}}
\end{align*}

\begin{figure}
    \centering
    \includegraphics[scale=0.39]{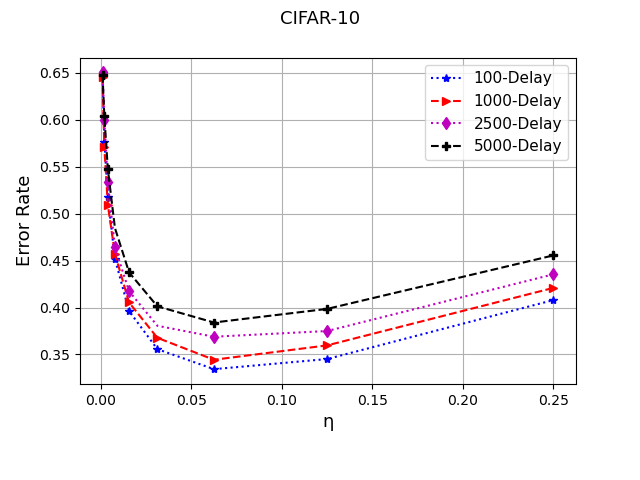} \\
    \includegraphics[scale=0.39]{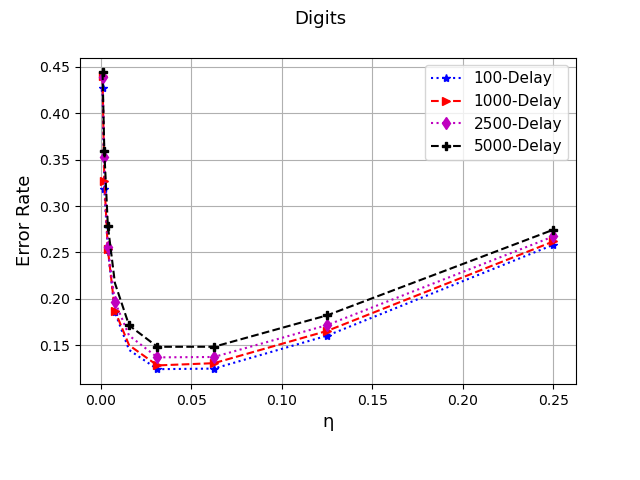} \\
    \includegraphics[scale=0.39]{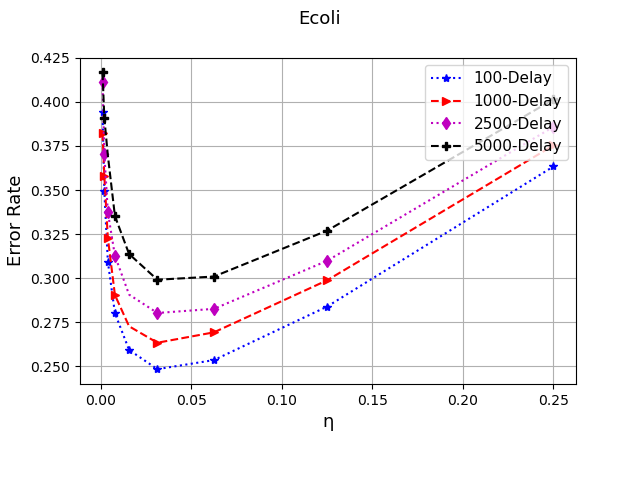} \\
    \includegraphics[scale=0.39]{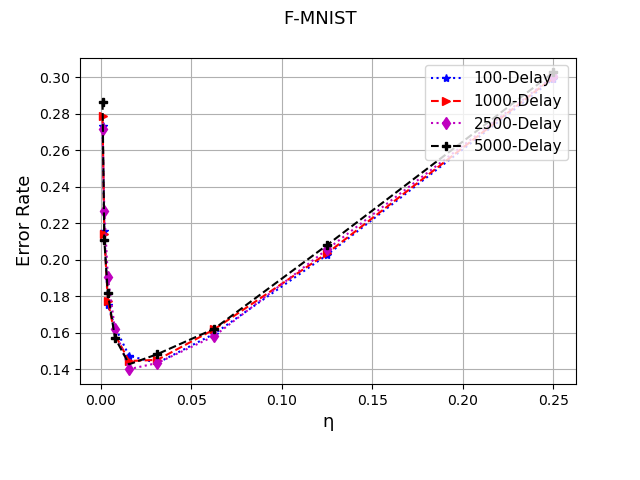} \\
    \includegraphics[scale=0.39]{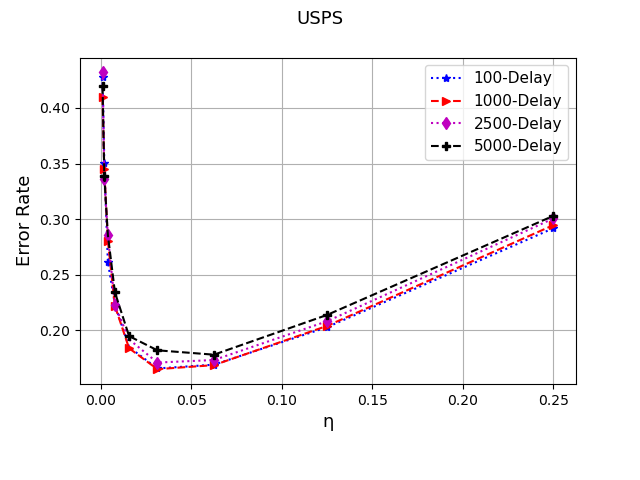} %
    \caption{Final error rates of Delaytron as the function of exploration parameter $\gamma$ for four different values of delay ($D=100, 1000, 2500$ and $5000$) for CIFAR-10, Digits, Ecoli, F-MNIST and USPS datasets}%
    \label{fig:gamma_plots}
\end{figure}

The maximum value of $\sum_{e=1}^ET_e2^{-\frac{e}{2}}$ (subject to $\sum_{e=1}^ET_e=T$) is achieved when $E=\lceil \log_2T\rceil$ and $e$-th epoch having length $2^e$. So, $\sum_{e=1}^ET_e2^{-\frac{e}{2}}\leq \sum_{e=1}^{\lceil \log_2T\rceil}2^{\frac{e}{2}}\leq \sqrt{2}\frac{2^{\frac{\lceil \log_2T\rceil}{2}}-1}{\sqrt{2}-1}\leq 5\sqrt{T}$. We also note that $\sum_{t=1}^Td_t\geq \sum_{t=1}^T\min\{d_t,T-t-1\}=\sum_{t=1}^Tm_t\geq \sum_{t=1}^{T_E}m_t\geq 2^E-1$. Thus,
$2^{\frac{E}{2}}\leq \sqrt{1+\sum_{t=1}^Td_t}\leq \sqrt{2\sum_{t=1}^Td_t}$.
Also, using $\sqrt{2}-1\geq 0.4$, we get
\begin{align*}
    &\mathbb{E}[\mathcal{R}(T)]
    \leq\frac{5}{2}\sqrt{2\sum_{t=1}^Td_t}\left[\frac{1}{2}\Vert \W\Vert_F^2+\frac{KR^2}{\gamma}+L\right]+\frac{10\sqrt{T}KR^2}{\gamma}\\
    &\leq 10\sqrt{\sum_{t=1}^Td_t}\left[\frac{1}{2}\Vert \W\Vert_F^2+\frac{KR^2}{\gamma}+L\right]\\
    &\;\;\;\;+10\sqrt{T}\left[\frac{1}{2}\Vert \W\Vert_F^2+\frac{KR^2}{\gamma}+L\right]\\
    &\leq 10\left[\frac{1}{2}\Vert \W\Vert_F^2+\frac{KR^2}{\gamma}+L\right]\sqrt{T+\sum_{t=1}^Td_t}
\end{align*}

\end{proof}
We see that even for unknown $T$ and $\sum_{t\notin\mathcal{M}}d_t$, Adaptive Delaytron achieves a regret bound of $\mathcal{O}\left(\sqrt{T+\sum_{t=1}^Td_t}\right)$.

\begin{figure*}
    \centering
    \subfigure
    {{\includegraphics[width=0.39\linewidth]{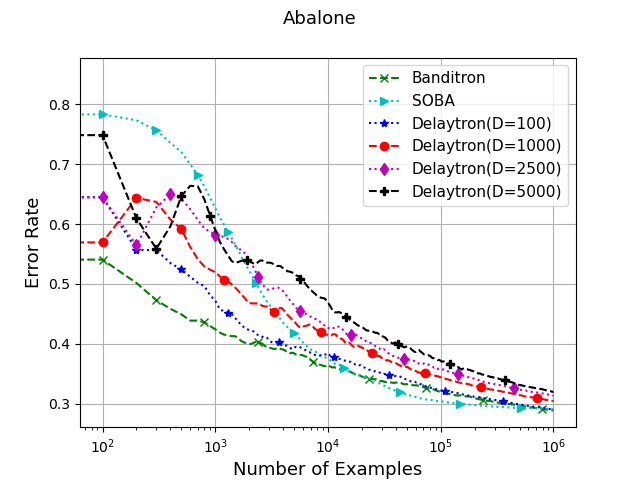} }}%
    \subfigure
    {{\includegraphics[width=0.39\linewidth]{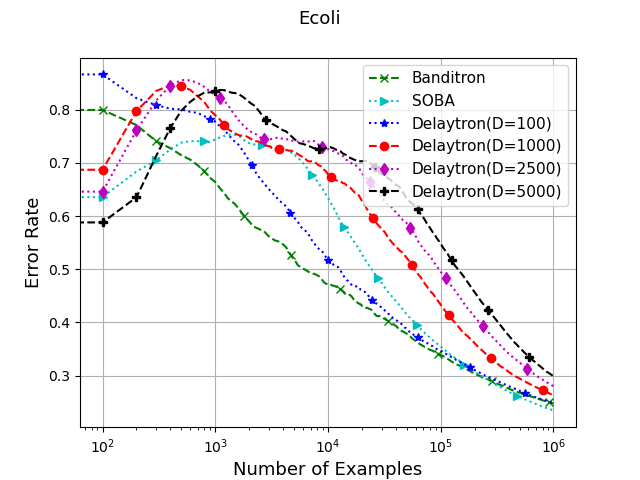} }}%
    \newline
    \subfigure
    {{\includegraphics[width=0.39\linewidth]{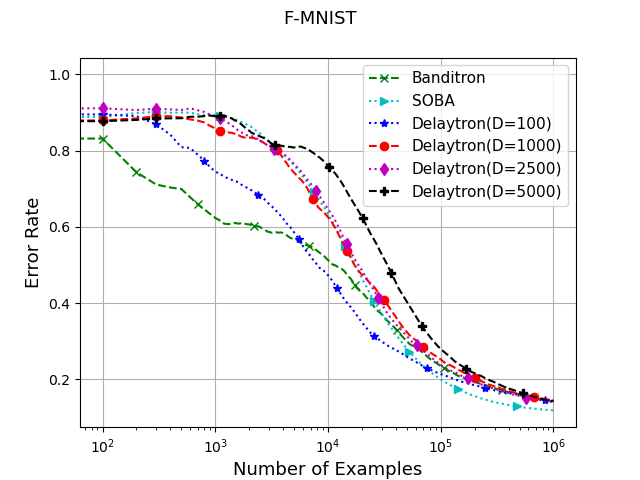} }}%
    \subfigure
    {{\includegraphics[width=0.39\linewidth]{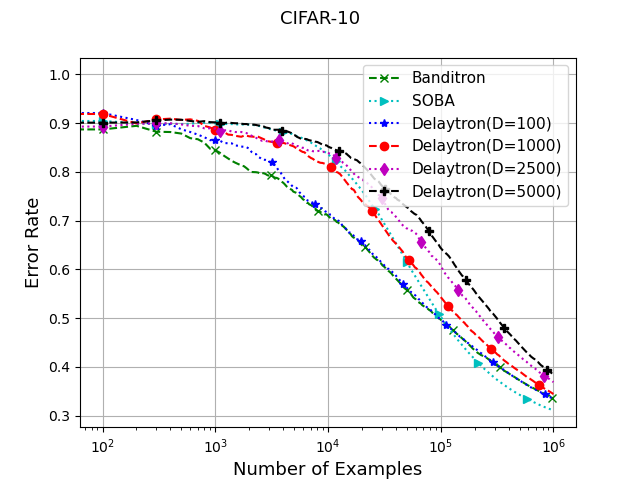} }}%
    \newline
    \subfigure
    {{\includegraphics[width=0.39\linewidth]{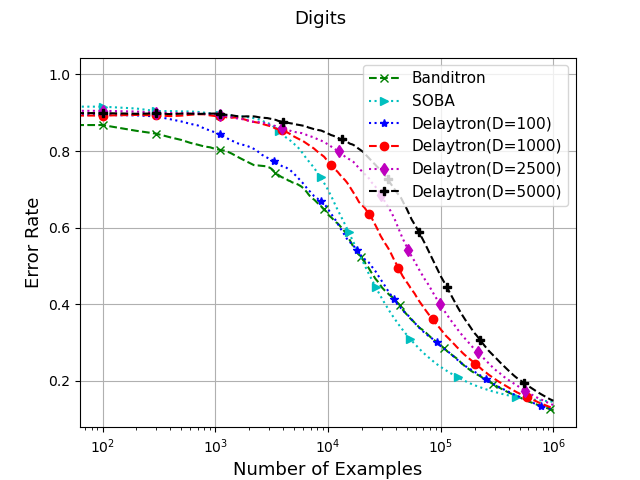} }}%
    \subfigure
    {{\includegraphics[width=0.39\linewidth]{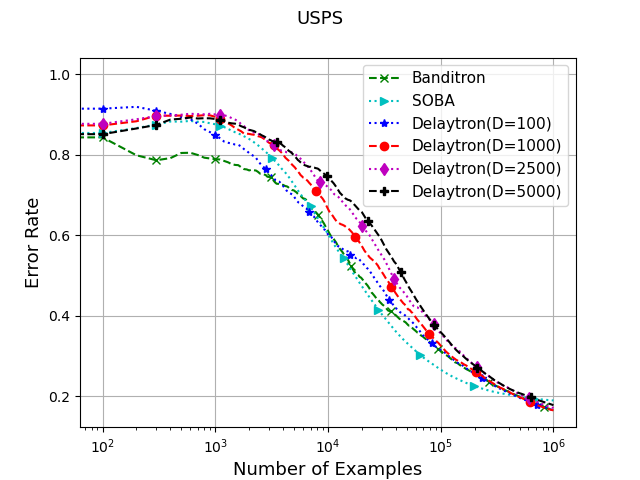} }}%
    \newline
    \subfigure
    {{\includegraphics[width=0.39\linewidth]{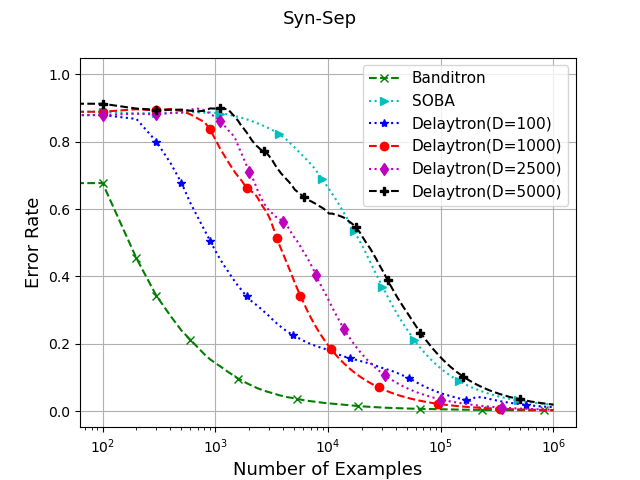} }}%
    \subfigure
    {{\includegraphics[width=0.39\linewidth]{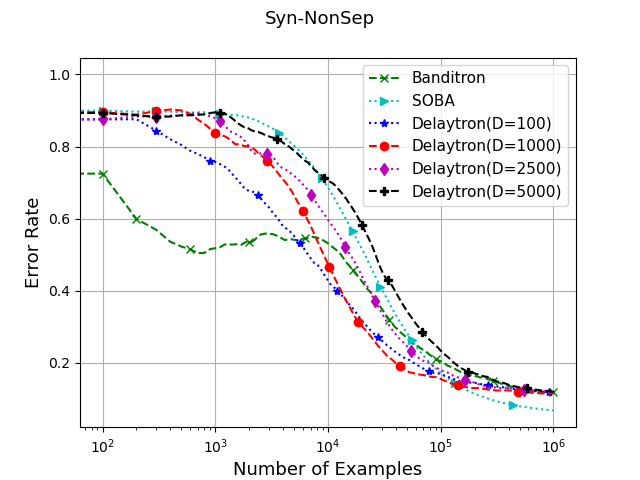} }}%
    \newline
     \caption{Average error rate of {\bf Delaytron} for four different values of delay ($D=100, 1000, 2500$ and $5000$) as compared against the other benchmark algorithm under the standard (0-Delay) setting.}%
    \label{fig:plots}
\end{figure*}

\begin{figure*}
    \centering
    \subfigure
    {{\includegraphics[width=0.39\linewidth]{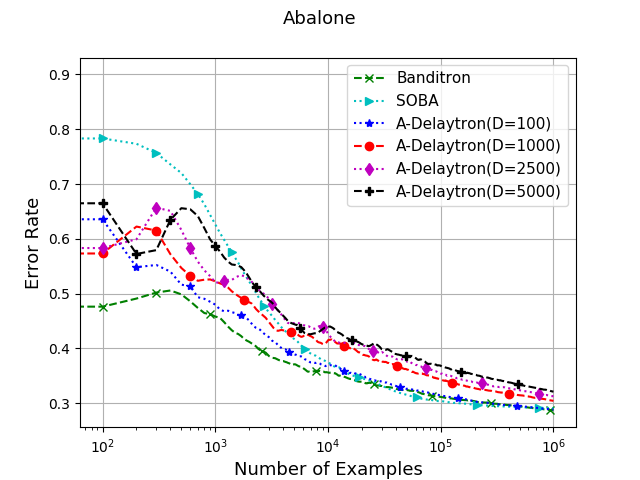} }}%
    \subfigure
    {{\includegraphics[width=0.39\linewidth]{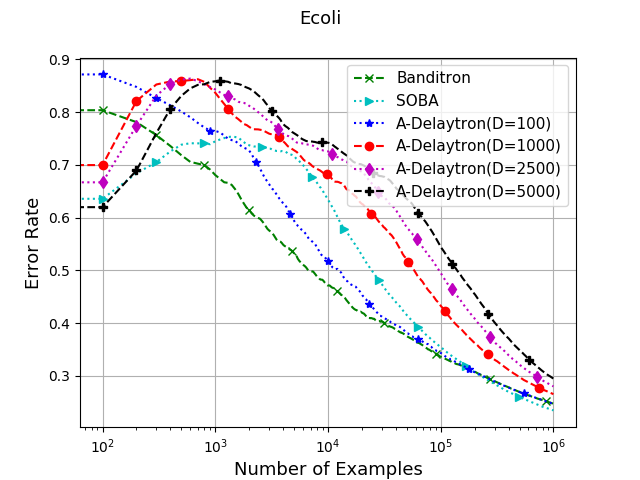} }}%
    \newline
    \subfigure
    {{\includegraphics[width=0.39\linewidth]{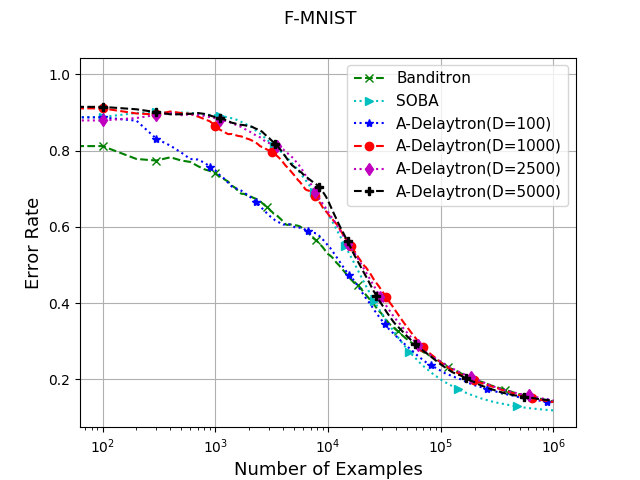} }}%
    \subfigure
    {{\includegraphics[width=0.39\linewidth]{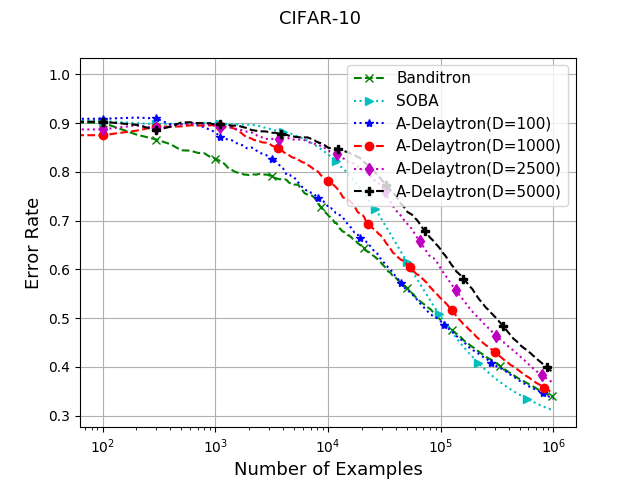} }}%
    \newline
    \subfigure
    {{\includegraphics[width=0.39\linewidth]{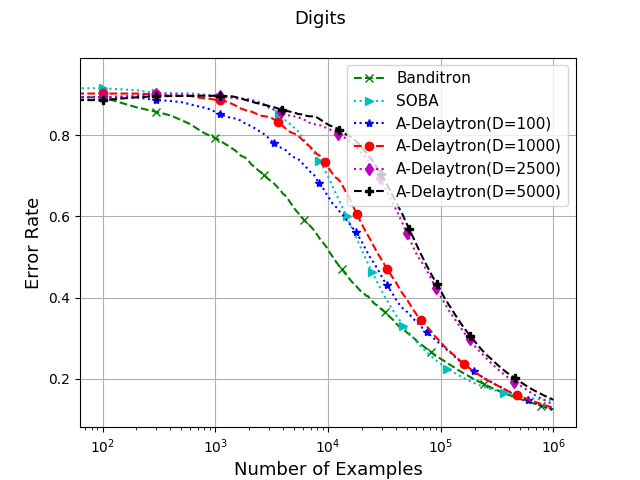} }}%
    \subfigure
    {{\includegraphics[width=0.39\linewidth]{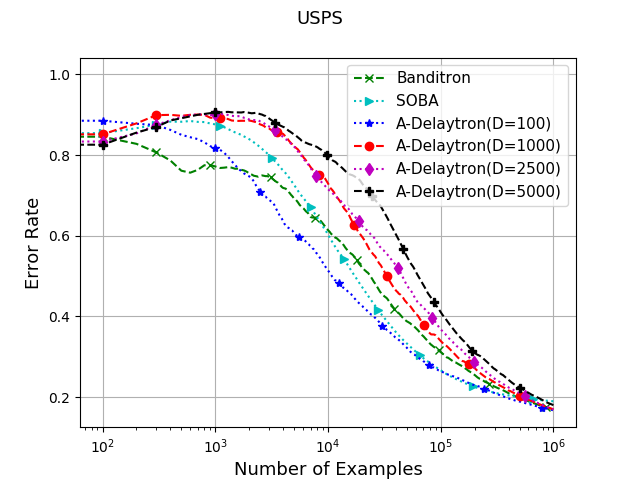} }}%
    \newline
    \subfigure
    {{\includegraphics[width=0.39\linewidth]{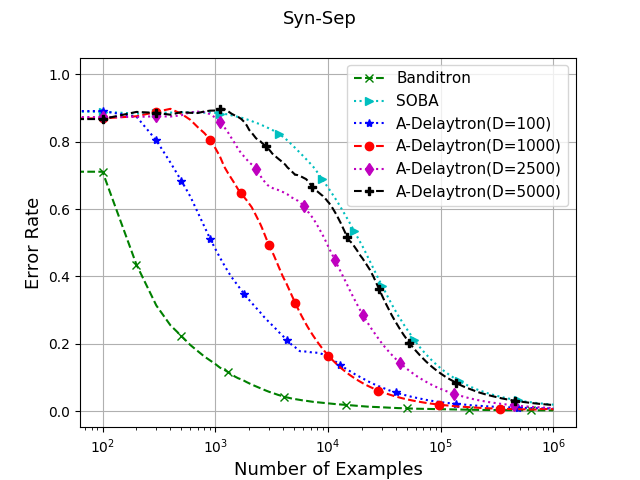} }}%
    \subfigure
    {{\includegraphics[width=0.39\linewidth]{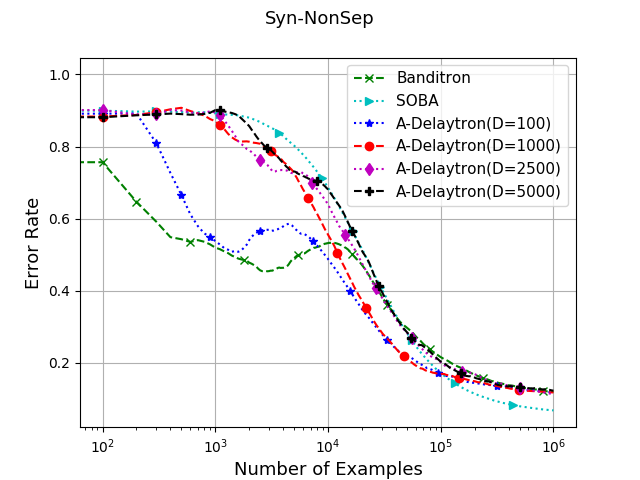} }}%
    \newline
     \caption{Average error rate of {\bf Adaptive Delaytron} for four different values of delay ($D=100, 1000, 2500$ and $5000$) as compared against the other benchmark algorithm under the standard (0-Delay) setting.}%
    \label{fig:plots2}
\end{figure*}

\section{Experiments}
This section empirically evaluates the effectiveness of our proposed approaches, Delaytron and Adaptive Delaytron, on various real-world and synthetic data sets.

{\bf Datasets Used: }We use CIFAR-10 \cite{krizhevsky2009learning}, Fashion-MNIST \cite{xiao2017fashion}, USPS \cite{hull1994database}, Digits, Abalone and Ecoli datasets from UCI repository \cite{Dua:2019}. To extract the features of examples in Fashion-MNIST dataset, we use a four-layer convolutional neural network as described in \cite{10.1007/978-3-030-75765-6_36}. Moreover, for CIFAR-10, we used a pre-trained VGG-16 \cite{simonyan2014very} model to extract features. 

We also performed experiments on synthetic datasets called SynSep and SynNonSep. SynSep is a 9-class, 400-dimensional synthetic data set of size $10^5$. While constructing SynSep, we ensure that the dataset is linearly separable. For more details about the dataset, please refer to \cite{10.1145/1390156.1390212}. The idea behind SynSep is to generate a simple dataset simulating a text document. The coordinates represent different words in a small vocabulary of size 400. SynNonSep is constructed the same way as SynSep except that a 5\% label noise is introduced, making the dataset non-separable.

{\bf Baselines: }We compared our proposed algorithm, Delaytron, under the delayed feedback setting with the present state-of-the-art bandit algorithms (Banditron \cite{10.1145/1390156.1390212}, and SOBA \cite{pmlr-v70-beygelzimer17a}) under the standard (0-Delay) setting. 

{\bf Experimental Settings and Hyper-parameter Selection: }We ran Delaytron for four different values of maximum delay (denoted by $D$). These are (a) 100 (b) 1000 (c) 2500 (d) 5000. We randomly sample the delay in the range $[0, D]$ during training at each trial. We ran Delaytron and other benchmarking algorithms for a wide range of the exploration parameter $\gamma$ values for each dataset and delay. Figure~\ref{fig:gamma_plots} plots the final error rates of Delaytron for different dataset as the function of exploration parameter $\gamma$. We chose the $\gamma$ value for which the average error rate achieved is minimum. Here, the averaging is done over 20 independent runs of the algorithm.

{\bf Results:}
Figure~\ref{fig:plots} and Figure~\ref{fig:plots2} show the plots of average error rates of Delaytron and Adaptive Delaytron (for the best value of parameter $\gamma$) against the number of instances observed so far. We plotted the result on a log-log scale to better visualize the asymptotic bounds. 

By observing Figure~\ref{fig:plots} and Figure~\ref{fig:plots2}, we see that as the number of trials (rounds) grows, the slope of the error rate of Delaytron under different delay settings is comparable to that of other benchmarking algorithms under the 0-Delay setting. The final error rate of Delaytron under various delays is also close to the 0-Delay approaches (e.g., Banditron and SOBA).

Figure~\ref{fig:plots2} shows that the performance of Adaptive Delaytron under different delay settings is comparable to that of other benchmarking algorithms under the 0-Delay setting. As a result, we can successfully learn a multiclass classifier under delayed bandit response settings without a priori knowledge of $T$.

We also observe that as the delay increases, Delaytron and Adaptive Delaytron take more rounds to converge. However, after a sufficient number of trials, the error curve of Delaytron meets the error curve of Banditron. This happens because our proposed algorithm successfully learns a multiclass classifier despite the delay in the received bandit feedback.

\section{Conclusion and Future Work}

In this paper, we proposed Delaytron algorithm that can efficiently learn multiclass classifiers with delayed bandit feedback. We show that the regret bound of the proposed approach for fixed step size is $\mathcal{O}\left(\sqrt{\frac{2 K}{\gamma}\left[\frac{T}{2}+\left(2+\frac{L^2}{R^2\Vert \W\Vert_F^2}\right)\sum_{t=1}^Td_t\right]}\right)$. Here, we assumed that the loss for each missing sample is upper bounded by $L$. On the other hand, when the loss for missing samples is not upper bounded, the regret bound of Delaytron is $\mathcal{O}\left(\sqrt{\frac{2 K}{\gamma}\left[\frac{T}{2}+2\sum_{t=1}^Td_t+\vert \mathcal{M}\vert T\right]}\right)$ where $\mathcal{M}$ is the set of missing samples in $T$ rounds. Note that the constant step size used in these bounds require the knowledge of $T$ and $\sum_{t=1}^Td_t$. When $T$ and $\sum_{t=1}^Td_t$ are unknown, we propose Adaptive Delaytron using a doubling trick for online learning and show that it achieves a regret bound of $\mathcal{O}\left(\sqrt{T+\sum_{t=1}^Td_t}\right)$. Experimental results show that Delaytron and Adaptive Delaytron with delayed bandit feedback performs comparable to the state of the art bandit feedback approaches which don't accept delays. 

 \bibliographystyle{plain}
\bibliography{example_paper}

\begin{thebibliography}{10}

\bibitem{10.1007/978-3-030-75765-6_36}
Mudit Agarwal and Naresh Manwani.
\newblock Learning multiclass classifier under noisy bandit feedback.
\newblock In {\em Advances in Knowledge Discovery and Data Mining}, pages
  448--460, 2021.

\bibitem{pmlr-v129-arora20a}
Maanik Arora and Naresh Manwani.
\newblock Exact passive-aggressive algorithms for multiclass classification
  using bandit feedbacks.
\newblock In {\em Proceedings of The 12th Asian Conference on Machine
  Learning}, volume 129 of {\em Proceedings of Machine Learning Research},
  pages 369--384, 18--20 Nov 2020.

\bibitem{10.1007/978-3-030-89188-6_5}
Gaurav Batra and Naresh Manwani.
\newblock Multiclass classification using dilute bandit feedback.
\newblock In {\em PRICAI 2021: Trends in Artificial Intelligence}, pages
  63--75, 2021.

\bibitem{pmlr-v70-beygelzimer17a}
Alina Beygelzimer, Francesco Orabona, and Chicheng Zhang.
\newblock Efficient online bandit multiclass learning with
  $\tilde{O}(\sqrt{T})$ regret.
\newblock In {\em Proceedings of the 34th International Conference on Machine
  Learning}, volume~70 of {\em Proceedings of Machine Learning Research}, pages
  488--497, 06--11 Aug 2017.

\bibitem{NEURIPS2019_ae2a2db4}
Ilai Bistritz, Zhengyuan Zhou, Xi~Chen, Nicholas Bambos, and Jose Blanchet.
\newblock Online exp3 learning in adversarial bandits with delayed feedback.
\newblock In {\em Advances in Neural Information Processing Systems},
  volume~32, 2019.

\bibitem{pmlr-v49-cesa-bianchi16}
Nicol‘o Cesa-Bianchi, Claudio Gentile, Yishay Mansour, and Alberto Minora.
\newblock Delay and cooperation in nonstochastic bandits.
\newblock In {\em 29th Annual Conference on Learning Theory}, volume~49, pages
  605--622, 23--26 Jun 2016.

\bibitem{Crammer2002}
Koby Crammer, Yoram Singer, Nello Cristianini, John Shawe-Taylor, and Bob
  Williamson.
\newblock On the algorithmic implementation of multiclass kernel-based vector
  machines.
\newblock {\em J. Mach. Learn. Res}, 2, 01 2002.

\bibitem{Dua:2019}
Dheeru Dua and Casey Graff.
\newblock {UCI} machine learning repository, 2017.

\bibitem{Hazan2011NewtronAE}
Elad Hazan and Satyen Kale.
\newblock Newtron: an efficient bandit algorithm for online multiclass
  prediction.
\newblock In {\em NIPS}, 2011.

\bibitem{10.1145/2648584.2648589}
Xinran He, Junfeng Pan, Ou~Jin, Tianbing Xu, Bo~Liu, Tao Xu, Yanxin Shi,
  Antoine Atallah, Ralf Herbrich, Stuart Bowers, and Joaquin Qui\~{n}onero
  Candela.
\newblock Practical lessons from predicting clicks on ads at facebook.
\newblock In {\em Proceedings of the Eighth International Workshop on Data
  Mining for Online Advertising}, ADKDD'14, page 1–9, 2014.

\bibitem{hull1994database}
Jonathan~J. Hull.
\newblock A database for handwritten text recognition research.
\newblock {\em IEEE Transactions on pattern analysis and machine intelligence},
  16(5):550--554, 1994.

\bibitem{10.1145/1390156.1390212}
Sham~M. Kakade, Shai Shalev-Shwartz, and Ambuj Tewari.
\newblock Efficient bandit algorithms for online multiclass prediction.
\newblock In {\em Proceedings of the 25th International Conference on Machine
  Learning}, ICML '08, page 440–447, 2008.

\bibitem{krizhevsky2009learning}
Alex Krizhevsky, Geoffrey Hinton, et~al.
\newblock Learning multiple layers of features from tiny images.
\newblock 2009.

\bibitem{Lim2010TheDO}
Kian~Ping Lim and Chee-Wooi Hooy.
\newblock The delay of stock price adjustment to information: A country- level
  analysis.
\newblock {\em Economics Bulletin}, 30:1609--1616, 2010.

\bibitem{10.5555/2371238}
Mehryar Mohri, Afshin Rostamizadeh, and Ameet Talwalkar.
\newblock {\em Foundations of Machine Learning}.
\newblock The MIT Press, 2012.

\bibitem{Ong2018DelayIR}
Mei-Sing Ong, Farah Magrabi, and Enrico~W. Coiera.
\newblock Delay in reviewing test results prolongs hospital length of stay: a
  retrospective cohort study.
\newblock {\em BMC Health Services Research}, 18, 2018.

\bibitem{NIPS2015_72da7fd6}
Kent Quanrud and Daniel Khashabi.
\newblock Online learning with adversarial delays.
\newblock In {\em Advances in Neural Information Processing Systems},
  volume~28, 2015.

\bibitem{simonyan2014very}
Karen Simonyan and Andrew Zisserman.
\newblock Very deep convolutional networks for large-scale image recognition.
\newblock {\em arXiv preprint arXiv:1409.1556}, 2014.

\bibitem{xiao2017fashion}
Han Xiao, Kashif Rasul, and Roland Vollgraf.
\newblock Fashion-mnist: a novel image dataset for benchmarking machine
  learning algorithms.
\newblock {\em arXiv preprint arXiv:1708.07747}, 2017.

\end{thebibliography}

\vfill

\end{document}